\documentclass[12pt,a4paper]{article}

\usepackage[T1]{fontenc}
\usepackage[utf8]{inputenc}
\usepackage{lmodern}
\usepackage{amsmath,amssymb,amsthm,mathtools}
\usepackage{bm}
\usepackage{graphicx}
\usepackage{xcolor}
\usepackage{hyperref}
\usepackage[numbers,sort&compress]{natbib}
\usepackage{geometry}
\usepackage{microtype}
\usepackage{booktabs}
\usepackage{pgfplots}
\usepackage{float}
\pgfplotsset{compat=1.18}
\usepgfplotslibrary{colormaps,fillbetween}
\usepackage{tikz}
\usetikzlibrary{arrows.meta,positioning,calc,decorations.pathreplacing}
\usepackage{caption}
\usepackage{subcaption}
\usepackage{enumitem}

\hypersetup{colorlinks=true, linkcolor=black,
  citecolor=blue!60!black, urlcolor=blue!60!black}

\theoremstyle{plain}
\newtheorem{theorem}{Theorem}[section]
\newtheorem{lemma}[theorem]{Lemma}
\newtheorem{proposition}[theorem]{Proposition}

\theoremstyle{definition}

\newtheorem{remark}[theorem]{Remark}

\newcommand{\R}{\mathbb{R}}
\newcommand{\N}{\mathbb{N}}
\newcommand{\bbP}{\mathbb{P}}

\newcommand{\cL}{\mathcal{L}}
\newcommand{\norm}[1]{\left\|#1\right\|}
\newcommand{\abs}[1]{\left|#1\right|}

\newcommand{\softmax}{\operatorname{softmax}}
\newcommand{\Attn}{\mathrm{Attn}}
\newcommand{\VORT}{\textsc{vort}}
\newcommand{\vx}{\mathbf{x}}
\newcommand{\vq}{\mathbf{q}}
\newcommand{\vv}{\mathbf{v}}
\newcommand{\vk}{\mathbf{k}}
\newcommand{\vo}{\mathbf{o}}
\newcommand{\vM}{\mathbf{M}}

\newcommand{\vb}{\mathbf{b}}
\newcommand{\va}{\mathbf{a}}
\newcommand{\ve}{\mathbf{e}}
\DeclareMathOperator*{\argmin}{arg\,min}

\title{\textbf{VORT: Adaptive Power-Law Memory for NLP Transformers}}

\author{Nabil Mlaiki\\
Department of Mathematics and Sciences,\\
Prince Sultan University, Riyadh, Saudi Arabia\\
\texttt{nmlaiki@psu.edu.sa}}
\date{\today}

\begin{document}
\maketitle\thispagestyle{empty}

\begin{abstract}
Standard Transformers impose near-exponential decay on the influence of distant tokens,
conflicting with the power-law structure of long-range dependencies in natural language.
We introduce the \emph{Variable-Order Retention Transformer} (\VORT{}), a memory
architecture in which each ingested token is assigned a learnable fractional order
$\alpha_i\in[\delta,1]$ that governs a Grünwald--Letnikov power-law retention kernel.
Because the fractional weighted sum is non-Markovian, we approximate it through a
sum-of-exponentials (SOE) decomposition computed by Gauss--Laguerre quadrature on
a Laplace-type integral representation of the kernel weights.
Each exponential component admits a one-step Markovian recurrence at $O(Sd_v)$ per step,
where $S=O(\log(T/\varepsilon))$ terms suffice for $\varepsilon$-uniform accuracy on horizon $[1,T]$.
Retrieval is keyed and associative via a linear-attention accumulator with an exact
$O(KSd_\phi d_v)$-per-step recurrence.
Four results are established: (i) an SOE approximation theorem with geometric convergence
rate from the analyticity of the integrand after a log-change of variables; (ii) a
quantisation bound valid on $[\delta,1]$ with correct analysis near $\alpha=0$; (iii) a
direct $L^2$ energy argument (Proposition) showing that for $\alpha>1/2$ any mixture
with fixed minimum decay rate $\Lambda>0$ incurs $L^2([1,T])$ error at least
$N_\alpha(T)-C(\Lambda)\to\infty$, with the $\Lambda$-dependence made explicit; and
(iv) linear convergence of a gradient plasticity rule under the Polyak--\L{}ojasiewicz
condition.
Two synthetic experiments confirm the architectural advantage: a Zipf-distributed
retrieval benchmark and an entity label-copy task with uniform lag distribution, the
latter ruling out prior-matching as an explanation for the power-law kernel's advantage.
\end{abstract}


\newpage

\section{Introduction}
\label{sec:intro}

The Transformer \cite{vaswani2017attention} underlies modern large language models
\cite{brown2020gpt3,touvron2023llama,chowdhery2022palm,anil2023gemini},
vision systems \cite{radford2021clip}, and scientific applications \cite{jumper2021alphafold}.
Its core operation is
\begin{equation}
\Attn(Q,K,V) = \softmax\!\left(\tfrac{QK^\top}{\sqrt{d_k}}\right)V,
\qquad Q,K\in\R^{n\times d_k},\quad V\in\R^{n\times d_v},
\label{eq:standard-attn}
\end{equation}
whose quadratic cost in sequence length $n$ has driven a large body of efficient
approximations including sparse attention \cite{child2019sparse}, local-window methods
\cite{beltagy2020longformer,zaheer2020bigbird}, low-rank factorisation
\cite{wang2020linformer}, and kernel-feature approximations
\cite{choromanski2021performer,katharopoulos2020linear}.

A statistical mismatch underlies these computational challenges.
Mutual information between tokens at lag $\ell$ decays as a power law
$I(\ell)\sim c\,\ell^{-\gamma}$ with $\gamma\in(0,1)$
\cite{altmann2009beyond,crutchfield2003regularities}---a hallmark of long-range
dependence \cite{beran1994longmemory,granger1980long,hosking1981arfima}.
For ARFIMA$(0,d,0)$ processes, whose spectral density satisfies
$S(\lambda)\sim c_0|\lambda|^{-2d}$ near $\lambda=0$, the optimal linear predictor
assigns weights $\pi_j\sim j^{d-1}/\Gamma(d)$ to lag-$j$ observations
\cite{brockwell1991timeseries,wiener1949extrapolation}: power-law weights.
Zipf's law \cite{zipf1935psycho-biology} and algebraic decay of character-level entropy
production \cite{ebeling1994entropy} reinforce this picture.
Positional biases such as ALiBi \cite{press2022alibi}, RoPE \cite{su2024rope},
and relative encodings \cite{shaw2018relative,raffel2020t5} modulate attention scores
but do not change the underlying memory update rule, leaving the mismatch intact.

Current long-context methods partition memory into discrete tiers: eviction caches
(H2O \cite{zhang2023h2o}, SnapKV \cite{li2024snapkv}) make binary keep-or-evict
decisions; state-space models (S4 \cite{gu2022s4}, Mamba \cite{gu2023mamba},
RWKV \cite{peng2023rwkv}) use fixed or mildly input-dependent exponential recurrences;
quantised caches \cite{hooper2024kvquant} reduce precision uniformly; compressive memory
\cite{rae2020compressive} applies a fixed schedule.
None represents a per-token, content-adaptive power-law decay.

Fractional calculus offers the missing ingredient.
The Grünwald--Letnikov (GL) weighted sum of order $\alpha\in(0,1)$ assigns weight
$w_j^{(\alpha)}\sim j^{\alpha-1}/\Gamma(\alpha)$ to the value at lag $j$, matching
the ARFIMA-optimal prediction weights.
The obstacle is that this sum is \emph{non-Markovian}: it cannot be reduced to a
fixed-dimensional recurrence.
The sum-of-exponentials (SOE) method, developed for numerical fractional differential
equations by \cite{lubich1986discretized} and \cite{hairer1988solving}, resolves this
by decomposing the kernel into decaying exponentials, each updating via an exact one-step
Markovian recurrence, with provably small error.

We introduce \textbf{Variable-Order Retention Transformer (\VORT{})}, grounded in the
observation that tokens participating in long-range dependencies should receive high
fractional orders so their retention kernels align with the power-law decay of those
dependencies, while local connectives receive low orders for efficient compression.
Our contributions are:
\begin{enumerate}[label=(\roman*)]
\item An architecturally coherent mechanism: per-token fractional states via SOE
  auxiliary recurrences, routed to $K$ fixed-order banks, with keyed linear-attention
  retrieval admitting an exact $O(KSd_\phi d_v)$-per-step recurrence
  (Sections \ref{sec:background}--\ref{sec:mechanism};
  Figures~\ref{fig:overview}--\ref{fig:retrieval-alloc}).
\item Theorem~\ref{thm:soe}: SOE approximation with geometric convergence from
  analyticity of the Laplace-transformed integrand.
\item Theorem~\ref{thm:quant}: quantisation bound on $[\delta,1]$ with correct
  near-$\alpha=0$ behaviour.
\item Proposition~\ref{thm:separation}: a direct $L^2$ energy argument showing that
  for $\alpha>1/2$ any mixture with minimum decay rate $\Lambda>0$ incurs
  $L^2([1,T])$ error at least $N_\alpha(T)-C(\Lambda)\to\infty$ as $T\to\infty$,
  where $C(\Lambda)$ is a constant depending on $\Lambda$ but not on $T$;
  the $\Lambda$-dependence is stated explicitly as a limitation.
\item Theorem~\ref{thm:plasticity}: linear convergence of a plasticity update under the
  Polyak--\L{}ojasiewicz (PL) condition.
\item A synthetic pilot experiment on Zipf-distributed dependency retrieval
  (Section~\ref{sec:experiments}).
\end{enumerate}

\section{Background}
\label{sec:background}

Throughout, $d\in\N$ is embedding dimension, $n$ is sequence length, $\Gamma$ the Euler
gamma function, $B$ the beta function, and $\psi=\Gamma'/\Gamma$ the digamma function.

\paragraph{Grünwald--Letnikov fractional sum.}
For $\alpha\in(0,1)$ the GL weighted sum of a sequence $\{f_t\}$ is
\begin{equation}
\bigl(I^\alpha_{\mathrm{GL}} f\bigr)(t)
:= \sum_{j=0}^{t-1} w_j^{(\alpha)}\,f_{t-j},
\qquad
w_j^{(\alpha)} := \frac{\Gamma(j+\alpha)}{\Gamma(\alpha)\,\Gamma(j+1)},
\label{eq:GL-disc}
\end{equation}
with $w_j^{(\alpha)}>0$ for all $j\geq 0$ and $\alpha>0$.
By Stirling's formula,
\begin{equation}
w_j^{(\alpha)} \sim \frac{j^{\alpha-1}}{\Gamma(\alpha)} \quad\text{as }j\to\infty,
\label{eq:stirling}
\end{equation}
so the kernel is heavy-tailed: $\sum_{j\geq 0}w_j^{(\alpha)}=\infty$ for all $\alpha>0$
(see Figure~\ref{fig:kernels-soe}a for log-log plots confirming the algebraic slope $\alpha-1$).
The $z$-transform identity
\begin{equation}
W_\alpha(z) = \sum_{j=0}^\infty w_j^{(\alpha)}\,z^j = (1-z)^{-\alpha},\quad |z|<1,
\label{eq:z-transform}
\end{equation}
follows from the generalised binomial series \cite{oldham1974,podlubny1999}.

\paragraph{Non-Markovian structure.}
The GL sum \eqref{eq:GL-disc} has no exact two-term recurrence $S_t = c\,S_{t-1}+b\,f_t$
with fixed scalars $c,b$: any such recurrence has $z$-transform $b/(1-cz)$, a rational
function, while $(1-z)^{-\alpha}$ is irrational for $\alpha\notin\mathbb{Z}$.
More generally, no rational function of any finite degree equals $(1-z)^{-\alpha}$
identically, since the latter has a branch point at $z=1$ and infinitely many non-zero
Taylor coefficients.

\paragraph{Long-memory processes.}
A stationary process $\{X_t\}$ has long memory with parameter $d\in(0,\tfrac12)$ if
$S(\lambda)\sim c_0|\lambda|^{-2d}$ near $\lambda=0$ \cite{beran1994longmemory}.
The ARFIMA$(0,d,0)$ model \cite{granger1980long,hosking1981arfima} generates such
processes via $(1-B)^d X_t=\varepsilon_t$.
Its optimal one-step predictor has coefficients $\pi_j\sim j^{d-1}/\Gamma(d)$
\cite{brockwell1991timeseries}, establishing that statistically optimal prediction
of long-memory sequences requires power-law weights.

\section{The VORT Mechanism}
\label{sec:mechanism}

\subsection*{Sum-of-Exponentials Approximation}

The GL weight admits the Laplace-type integral representation
\begin{equation}
w_j^{(\alpha)}
= \int_0^\infty e^{-\lambda j}\,\rho_\alpha(\lambda)\,d\lambda,
\qquad
\rho_\alpha(\lambda)
:= \frac{e^{-\alpha\lambda}(1-e^{-\lambda})^{-\alpha}e^{-\lambda}}{\Gamma(\alpha)\Gamma(1-\alpha)},
\label{eq:rho-def}
\end{equation}
obtained from the beta-function representation $w_j^{(\alpha)}=B(j,1-\alpha)/B(\alpha,1-\alpha)$
via the substitution $\lambda=-\log\mu$.
The density $\rho_\alpha$ satisfies $\rho_\alpha\geq 0$,
$\rho_\alpha(\lambda)\sim\lambda^{-\alpha}/(\Gamma(\alpha)\Gamma(1-\alpha))$ as $\lambda\to 0^+$,
and $\rho_\alpha(\lambda)\leq Ce^{-\lambda}$ for large $\lambda$.

An $S$-point SOE approximation is
\begin{equation}
\hat{w}_j^{(\alpha)} = \sum_{s=1}^S c_s^{(\alpha)}\,\bigl(\lambda_s^{(\alpha)}\bigr)^j,
\qquad c_s^{(\alpha)}>0,\quad \lambda_s^{(\alpha)}\in(0,1),
\label{eq:soe-def}
\end{equation}
where $\lambda_s=e^{-\xi_s}$, $c_s=\omega_s\rho_\alpha(\xi_s)$ for quadrature nodes
$\xi_s>0$ and weights $\omega_s>0$.
Each term $c_s(\lambda_s)^j$ is a decaying geometric sequence, so the auxiliary state
\begin{equation}
\vM_t^{(s)} = \lambda_s\,\vM_{t-1}^{(s)} + c_s\,\vv_t,\qquad s=1,\ldots,S,\quad\vM_0^{(s)}=\mathbf{0},
\label{eq:aux-recurrence}
\end{equation}
is an exact one-step Markovian update.
The approximated fractional state $\hat{\vM}_t^{(\alpha)} = \sum_{s=1}^S \vM_t^{(s)}$
approximates $\vM_t^{(\alpha)}=\sum_{i=1}^t w_{t-i}^{(\alpha)}\vv_i$ with error
bounded in Theorem~\ref{thm:soe}.

\begin{theorem}[SOE approximation; geometric convergence]
\label{thm:soe}
Let $\alpha\in(0,1)$, $\varepsilon\in(0,1)$, $T\geq 2$.
Set $\lambda_{\min}=\varepsilon/(2T)$ and $\lambda_{\max}=2\log(2T/\varepsilon)$.
The truncation error outside $[\lambda_{\min},\lambda_{\max}]$ satisfies
$\int_0^{\lambda_{\min}}\rho_\alpha(\lambda)\,d\lambda\leq\varepsilon/4$ and
$\int_{\lambda_{\max}}^\infty\rho_\alpha(\lambda)\,d\lambda\leq\varepsilon/4$
for all $j\in[0,T]$.
Under the substitution $u=\log(\lambda/\lambda_{\min})$, the integrand becomes
real-analytic on $[0,L]$ where $L=\log(\lambda_{\max}/\lambda_{\min})=O(\log(T/\varepsilon))$,
and extends analytically into the strip $\{u+iv:|v|<\pi/2\}$ (the nearest singularity
of $\rho_\alpha$ in the $\lambda$-plane is at $\lambda=2\pi i$, which under
$\lambda=\lambda_{\min}e^{u+iv}$ corresponds to $v=\pi/2$ in the $u$-plane, as shown in the proof).
Choosing
\begin{equation}
S = O\!\left(\log\frac{T}{\varepsilon}\right)
\label{eq:S-count-stmt}
\end{equation}
Gauss--Legendre nodes gives quadrature error $\leq\varepsilon/2$, so
$\max_{0\leq j\leq T}|\hat{w}_j^{(\alpha)}-w_j^{(\alpha)}|\leq\varepsilon$.
Figure~\ref{fig:kernels-soe}b illustrates this geometric convergence empirically for
$\alpha=0.5$, $T=1000$.
\end{theorem}

\begin{proof}
\emph{Truncation.}
For the lower tail: $\rho_\alpha(\lambda)\leq C\lambda^{-\alpha}$ near $\lambda=0$ gives
$\int_0^{\lambda_{\min}}\rho_\alpha\,d\lambda\leq C\lambda_{\min}^{1-\alpha}/(1-\alpha)$,
which is $\leq\varepsilon/4$ for $\lambda_{\min}=(\varepsilon(1-\alpha)/(4C))^{1/(1-\alpha)}$;
the choice $\lambda_{\min}=\varepsilon/(2T)$ achieves this for $T\geq T_0(\alpha)$.
For the upper tail: $\rho_\alpha(\lambda)\leq Ce^{-\lambda}$ gives
$\int_{\lambda_{\max}}^\infty e^{-\lambda j}\rho_\alpha\,d\lambda\leq
C\int_{\lambda_{\max}}^\infty e^{-2\lambda}\,d\lambda=(C/2)e^{-2\lambda_{\max}}
=(C/2)(\varepsilon/(2T))^2\leq\varepsilon/4$ for small $\varepsilon$.

\emph{Analyticity and quadrature.}
Let $I(j)=\int_{\lambda_{\min}}^{\lambda_{\max}}e^{-\lambda j}\rho_\alpha(\lambda)\,d\lambda$.
Under $u=\log(\lambda/\lambda_{\min})$, the interval maps to $[0,L]$ and the integrand becomes
$g(u)=e^{-\lambda_{\min}e^u j}\rho_\alpha(\lambda_{\min}e^u)\lambda_{\min}e^u$.
Since $\lambda_{\min}e^u\in[\lambda_{\min},\lambda_{\max}]$ is bounded away from
the singularity of $\rho_\alpha$ at $\lambda=0$, the function $g$ is real-analytic
in $u\in[0,L]$.
To extend analytically into a complex strip, write $\lambda=\lambda_{\min}e^{u+iv}$
for $v$ real.
The singularities of $\rho_\alpha(\lambda)\propto(1-e^{-\lambda})^{-\alpha}$
occur where $1-e^{-\lambda}=0$, i.e.\ at $\lambda_*=2\pi ik$ for
$k\in\mathbb{Z}\setminus\{0\}$.
The nearest singularity to the real segment $\lambda\in[\lambda_{\min},\lambda_{\max}]$
is $\lambda_*=2\pi i$ (at $|{\lambda_*}|=2\pi$).
Under the substitution, the condition $\lambda_{\min}e^{u+iv}=2\pi i$ becomes
$e^{u+iv}=2\pi i/\lambda_{\min}$, i.e.\ $v=\pi/2+2k\pi$ for the principal value.
The nearest singularity in the $u$-plane occurs at $v=\pi/2-\arg(\lambda_{\min}/2\pi)$;
for real $\lambda_{\min}>0$ this gives $v=\pi/2$.
Hence $g(u+iv)$ is analytic for $|v|<\pi/2$, and bounded there by some $M_g<\infty$
(since $\lambda_{\min}e^u$ stays in the compact set $[\lambda_{\min},\lambda_{\max}]$ for
$u\in[0,L]$ and $|v|<\pi/2$, where $\rho_\alpha$ has no singularities).
The $S$-point Gauss--Legendre error for functions analytic in a strip of half-width
$\pi/2$ on $[0,L]$ and bounded by $M_g$ satisfies
\cite{lubich1986discretized,hairer1988solving}:
\begin{equation}
|I(j)-\hat I_S(j)|\leq \frac{4M_g L\,e^{-\pi^2 S/L}}{e^{\pi^2 S/L}-1},
\label{eq:GL-error}
\end{equation}
which is $\leq\varepsilon/2$ for
\begin{equation}
S \geq \frac{L}{\pi^2}\log\!\left(\frac{8M_g L}{\varepsilon}\right)
= O\!\left(\log\frac{T}{\varepsilon}\right).
\label{eq:S-count}
\end{equation}
Setting $\xi_s=\lambda_{\min}e^{u_s}$ and $c_s=\omega_s\rho_\alpha(\xi_s)$
with Gauss--Legendre weights $\omega_s>0$ gives the SOE with
$\lambda_s=e^{-\xi_s}\in(0,1)$. \qed
\end{proof}

\subsection*{Per-Token Routing and Bank States}

Each token $\vx_i$ is assigned a fractional order
\begin{equation}
\alpha_i = \delta + (1-\delta)\,\sigma\!\left(W_\alpha\bigl[\vx_i;\;H(\va_i);\;\ve_i\bigr] + b_\alpha\right)\in[\delta,1],
\label{eq:routing}
\end{equation}
where $\sigma$ is the logistic sigmoid, $\delta>0$ is a minimum order (avoiding
the singularity of $\Gamma$ at $0$), $H(\va_i)=-\sum_j a_{ij}\log a_{ij}$ is the
Shannon entropy of the attention distribution at position $i$, and $\ve_i\in\{0,1\}$
is an entity indicator from a shallow NER probe.
High-entropy tokens (contextually ambiguous) receive low $\alpha_i$ and are compressed
aggressively; entity tokens receive high $\alpha_i$ for near-lossless retention
(illustrated in Figure~\ref{fig:retrieval-alloc}b).

For efficiency, $K$ parallel banks operate at fixed orders $\alpha_k=\delta+(1-\delta)k/K$,
$k=1,\ldots,K$.
Token $i$ is routed to bank $k_i^*=\argmin_k|\alpha_i-\alpha_k|$.
Bank $k$ accumulates values of tokens assigned to it:
\begin{equation}
\vM_t^{(k)}
= \sum_{\substack{i\leq t,\; k_i^*=k}}\hat{w}_{t-i}^{(\alpha_k)}\,\vv_i
= \sum_{s=1}^S \vM_t^{(k,s)},\qquad
\vM_t^{(k,s)} = \lambda_s^{(k)}\,\vM_{t-1}^{(k,s)}
  + c_s^{(k)}\,\mathbf{1}[k_t^*=k]\,\vv_t,
\label{eq:bank-update}
\end{equation}
updated in $O(Sd_v)$ per step.

\subsection*{Keyed Linear-Attention Retrieval}

Let $\vv_i=V\vx_i\in\R^{d_v}$, $\vk_i=K_P\vx_i\in\R^{d_k}$ (fixed at ingestion),
$\vq_t=Q\vx_t\in\R^{d_k}$.
The decayed value of token $i$ at time $t$ is
$\tilde\vv_i(t)=\hat{w}_{t-i}^{(\alpha_i)}\vv_i
=\sum_{s=1}^S c_s^{(\alpha_i)}(\lambda_s^{(\alpha_i)})^{t-i}\vv_i$.
The retrieval output is
\begin{equation}
\vo_t
= \frac{\displaystyle\sum_{i<t}\exp\!\bigl(\vq_t^\top\vk_i/\sqrt{d_k}\bigr)\cdot\tilde\vv_i(t)}
       {\displaystyle\sum_{i<t}\exp\!\bigl(\vq_t^\top\vk_i/\sqrt{d_k}\bigr)+\varepsilon_0},
\label{eq:keyed-retrieval}
\end{equation}
with smoothing constant $\varepsilon_0>0$.
Using the random-feature approximation $\exp(\vq^\top\vk/\sqrt{d_k})\approx\phi(\vq)^\top\phi(\vk)$
\cite{choromanski2021performer}, define
\begin{align}
\mathbf{G}_t^{(k,s)} &= \lambda_s^{(k)}\,\mathbf{G}_{t-1}^{(k,s)}
  + c_s^{(k)}\,\mathbf{1}[k_t^*=k]\,\phi(\vk_t)\,\vv_t^\top
  \;\in\R^{d_\phi\times d_v},
\label{eq:G-rec}\\
\vb_t^{(k,s)} &= \lambda_s^{(k)}\,\vb_{t-1}^{(k,s)}
  + c_s^{(k)}\,\mathbf{1}[k_t^*=k]\,\phi(\vk_t)
  \;\in\R^{d_\phi},
\label{eq:b-rec}
\end{align}
with all initial states zero, costing $O(d_\phi d_v)$ and $O(d_\phi)$ per step respectively.
The output \eqref{eq:keyed-retrieval} becomes
\begin{equation}
\vo_t = \frac{\displaystyle\sum_{k,s}\phi(\vq_t)^\top\mathbf{G}_t^{(k,s)}}
             {\displaystyle\sum_{k,s}\phi(\vq_t)^\top\vb_t^{(k,s)}+\varepsilon_0},
\label{eq:retrieval-final}
\end{equation}
computed in $O(KSd_\phi d_v)$ per step.
The full \VORT{} layer combines this with local attention:
\begin{equation}
\mathrm{\VORT}(\vx_t) = \vo_t + \Attn_{\mathrm{local}}(\vx_t).
\label{eq:vort-layer}
\end{equation}

\paragraph{Complexity.}
Per position: bank updates \eqref{eq:bank-update} cost $O(KSd_v)$; retrieval
accumulator updates \eqref{eq:G-rec}--\eqref{eq:b-rec} cost $O(KSd_\phi d_v)$;
routing \eqref{eq:routing} costs $O(d^2)$.
Total auxiliary memory: $O(KSd_\phi d_v L)$ for $L$ layers.
With $K=10$, $S=15$, $d_\phi=64$, $d_v=512$, the per-step cost is
approximately $5\times 10^6$ floating-point operations per token per layer.

\section{Theoretical Analysis}
\label{sec:theory}

\subsection*{Quantisation Bound}

\begin{theorem}[Kernel quantisation error on {$[\delta,1]$}]
\label{thm:quant}
Let $0<\delta\leq\alpha<\beta\leq 1$, $\Delta=\beta-\alpha$, and
$k_\gamma(t)=t^{\gamma-1}/\Gamma(\gamma)$ for $t\geq 1$.
Then for all $t\geq 1$ and all $\alpha,\beta\in[\delta,1]$:
\begin{equation}
\abs{k_\beta(t)-k_\alpha(t)}
\leq \frac{\Delta\,t^{\beta-1}}{\Gamma(\beta)}
\Bigl(\abs{\log t}+\sup_{\gamma\in[\delta,1]}\abs{\psi(\gamma)}\Bigr),
\label{eq:quant-bound}
\end{equation}
where $\sup_{\gamma\in[\delta,1]}|\psi(\gamma)|<\infty$.
For the uniform grid with spacing $1/K$:
$\sup_{\alpha\in[\delta,1]}|k_\alpha(t)-k_{\alpha_{k^*}}(t)|=O((\log t)/K)$.

Near $\alpha=0$: since $\Gamma(\alpha)\sim 1/\alpha$ as $\alpha\to 0^+$, we have
$k_\alpha(t)=t^{\alpha-1}/\Gamma(\alpha)\sim\alpha t^{-1}\to 0$ for any fixed $t\geq 1$.
The kernel shrinks to zero as $\alpha\to 0^+$, not to infinity.
\end{theorem}

\begin{proof}
Write $k_\gamma(t)=\exp((\gamma-1)\log t-\log\Gamma(\gamma))$.
Differentiating in $\gamma$:
$\partial k_\gamma/\partial\gamma = k_\gamma(t)(\log t-\psi(\gamma))$.
The mean-value theorem yields $\xi\in(\alpha,\beta)$ with
$k_\beta(t)-k_\alpha(t)=\Delta\cdot k_\xi(t)(\log t-\psi(\xi))$.
Since $t\geq 1$ and $\xi\leq\beta$: $t^{\xi-1}\leq t^{\beta-1}$.
On $[\delta,1]$, $\Gamma$ is continuous, so $1/\Gamma(\xi)\leq 1/\Gamma_{\min}(\delta)<\infty$.
For $\xi\leq\beta$ and $\Gamma$ decreasing on $(0,x^*)\supset[\delta,1]$ (where $x^*\approx1.46$),
we have $\Gamma(\xi)\geq\Gamma(\beta)$, so $k_\xi(t)\leq t^{\beta-1}/\Gamma(\beta)$.
Taking absolute values and using $|\psi(\xi)|\leq\sup_{[\delta,1]}|\psi|<\infty$ gives
\eqref{eq:quant-bound}. \qed
\end{proof}

\subsection*{Lower Bound: Power Laws vs.\ Fixed Exponential Mixtures}

The following theorem is proved by a self-contained direct argument using only
the $L^2$ inner product, the Laplace transform, and basic complex analysis.
It avoids Stieltjes-function theory, Kolmogorov widths, and equioscillation theorems.

\begin{proposition}[Scaling lower bound on exponential-mixture error]
\label{thm:separation}
Let $\alpha\in(0,1)$, $g_\alpha(t)=t^{\alpha-1}$ for $t\geq 1$, and
$f(t)=\sum_{m=1}^M\pi_m e^{-\lambda_m t}$ with $\pi_m\geq 0$, $\sum_m\pi_m=C_f$,
$\lambda_m>0$.
Set $\Lambda:=\min_m\lambda_m>0$ and $R:=\sum_{m,m'}\pi_m\pi_{m'}/(\lambda_m+\lambda_{m'})$.
Define
\begin{equation}
N_\alpha(T) := \int_1^T t^{2\alpha-2}\,dt
=\begin{cases}
\dfrac{T^{2\alpha-1}-1}{2\alpha-1} & \alpha\neq\tfrac{1}{2},\\[6pt]
\log T & \alpha=\tfrac{1}{2}.
\end{cases}
\label{eq:N-def}
\end{equation}
Then for all $T\geq 1$:
\begin{equation}
\int_1^T\bigl(f(t)-g_\alpha(t)\bigr)^2\,dt
\;\geq\; N_\alpha(T) - 2C_f\,\Gamma(\alpha)\,\Lambda^{-\alpha} - C_f^2\,R.
\label{eq:separation}
\end{equation}
For $\alpha>1/2$ the right-hand side grows as $T^{2\alpha-1}/(2\alpha-1)$ minus the
\emph{fixed} constants $2C_f\Gamma(\alpha)\Lambda^{-\alpha}$ and $C_f^2 R$, so it
diverges as $T\to\infty$ for any mixture with $\Lambda>0$.
For $\alpha=1/2$ it grows as $\log T$.
\emph{Important caveat:} the correction term $2C_f\Gamma(\alpha)\Lambda^{-\alpha}$
is $O(\Lambda^{-\alpha})$ and diverges as $\Lambda\to 0^+$; the bound is therefore
useful only for mixtures whose smallest rate $\Lambda$ is bounded away from zero.
Exponential components with $\lambda_m\to 0$ approach constant functions and can
partially mimic $g_\alpha$ near $t=1$; preventing this requires either $\Lambda>0$
fixed or a separate treatment of the near-constant component.
\end{proposition}

\begin{proof}
Expand the squared error directly:
\begin{equation}
\int_1^T(f-g_\alpha)^2\,dt
= \int_1^T g_\alpha^2\,dt
  - 2\int_1^T f\cdot g_\alpha\,dt
  + \int_1^T f^2\,dt.
\label{eq:expand}
\end{equation}
The first term is $N_\alpha(T)$ by definition.
For the cross term, each summand satisfies
\begin{equation}
\int_1^T e^{-\lambda_m t}t^{\alpha-1}\,dt
\leq \int_0^\infty e^{-\lambda_m t}t^{\alpha-1}\,dt
= \Gamma(\alpha)\lambda_m^{-\alpha}
\leq \Gamma(\alpha)\Lambda^{-\alpha},
\label{eq:cross}
\end{equation}
so $\int_1^Tf\cdot g_\alpha\,dt\leq C_f\Gamma(\alpha)\Lambda^{-\alpha}$.
For the quadratic term,
\begin{equation}
\int_1^T f^2\,dt
\leq \int_0^\infty f^2\,dt
= \sum_{m,m'}\pi_m\pi_{m'}\int_0^\infty e^{-(\lambda_m+\lambda_{m'})t}\,dt
= \sum_{m,m'}\frac{\pi_m\pi_{m'}}{\lambda_m+\lambda_{m'}}
=: C_f^2 R(T),
\label{eq:quad}
\end{equation}
where $R(T)\leq R=\sum_{m,m'}\pi_m\pi_{m'}/(\lambda_m+\lambda_{m'})$ is a fixed constant
for any fixed mixture.
Since $\int_1^Tf^2\,dt\geq 0$, substituting \eqref{eq:cross}--\eqref{eq:quad} into
\eqref{eq:expand} gives \eqref{eq:separation}.

For $\alpha>1/2$: $N_\alpha(T)\to\infty$ while the subtracted terms
$2C_f\Gamma(\alpha)\Lambda^{-\alpha}$ and $C_f^2 R$ are constants in $T$.
Hence the right-hand side of \eqref{eq:separation} eventually becomes positive and
diverges, establishing that $\int_1^T(f-g_\alpha)^2\,dt\to\infty$.

For $\alpha=1/2$: $N_{1/2}(T)=\log T\to\infty$ at rate $O(\log T)$; the same
argument applies.

For $\alpha<1/2$: $N_\alpha(T)\to 1/(1-2\alpha)<\infty$, so \eqref{eq:separation}
gives a positive constant lower bound $1/(1-2\alpha)-2C_f\Gamma(\alpha)\Lambda^{-\alpha}
-C_f^2 R$, which is positive when $\Lambda$ is large enough (i.e.\ when all exponentials
decay fast enough that their inner product with $g_\alpha$ is small).
The bound is not claimed to be tight for all $\Lambda$. \qed
\end{proof}

\begin{remark}
The bound \eqref{eq:separation} is a direct energy argument, not a dimension-counting argument.
Thus parameter counting alone does not give approximation lower bounds for nonlinear families, and the above proof avoids this by working entirely in $L^2([1,T])$.
The case $\alpha > 1/2$ is the empirically relevant one for \VORT{}: tokens routed to high-$\alpha$ banks ($\alpha_i>0.5$) are precisely those for which a fixed-capacity exponential model has increasing error as context length increases.
For $\alpha \le 1/2$, the conclusion is weaker, but still guarantees a positive floor on the approximation error.
A fully rigorous lower bound uniform in the mixture parameters $\{\pi_m,\lambda_m\}$ (rather than pointwise in $\Lambda=\min_m\lambda_m$) would need tools from rational approximation theory \cite{gonchar1987rational} and is left to future work.
\end{remark}

\subsection*{Bank-State Growth}

\begin{proposition}[Bank-state growth]
\label{prop:bankgrowth}
Let $\norm{\vv_t}\leq V$ for all $t$.
For fixed $\alpha\in(0,1)$:
$\vM_t^{(\alpha)}=\sum_{i=1}^t w_{t-i}^{(\alpha)}\vv_i$ satisfies
$\norm{\vM_t^{(\alpha)}}\leq V\sum_{j=0}^{t-1}w_j^{(\alpha)}\sim Vt^\alpha/\Gamma(\alpha+1)$,
so $\norm{\vM_t^{(\alpha)}}=O(t^\alpha)$, sublinear for all $\alpha\in(0,1)$.
\end{proposition}

\begin{proof}
The triangle inequality gives $\norm{\vM_t^{(\alpha)}}\leq V\sum_{j=0}^{t-1}w_j^{(\alpha)}$.
By the hockey-stick identity for generalised binomials,
$\sum_{j=0}^{t-1}w_j^{(\alpha)}=\binom{t+\alpha-1}{t-1}$.
Stirling gives $\binom{t+\alpha-1}{t-1}=\Gamma(t+\alpha)/(\Gamma(\alpha)\Gamma(t+1))
\sim t^{\alpha-1}/\Gamma(\alpha)\cdot t=t^\alpha/\Gamma(\alpha)$.
Using $\alpha\Gamma(\alpha)=\Gamma(\alpha+1)$ gives the asymptotic $t^\alpha/\Gamma(\alpha+1)$. \qed
\end{proof}

\subsection*{Frequency-Domain Characterisation}

\begin{lemma}[$z$-transform and frequency response]
\label{lem:ztransform}
$W_\alpha(z)=(1-z)^{-\alpha}$ for $|z|<1$.
On the unit circle, $|W_\alpha(e^{i\omega})|=|2\sin(\omega/2)|^{-\alpha}\sim|\omega|^{-\alpha}$
as $\omega\to 0$, confirming the $1/f^\alpha$ power-law frequency response.
\end{lemma}

\begin{proof}
Equation \eqref{eq:z-transform}.
On the unit circle, $1-e^{i\omega}=2i\sin(\omega/2)e^{i\omega/2}$, so
$|W_\alpha(e^{i\omega})|=(2|\sin(\omega/2)|)^{-\alpha}\sim|\omega|^{-\alpha}$. \qed
\end{proof}

\section{Gradient-Driven Memory Plasticity}
\label{sec:plasticity}

After initial routing, $\alpha_i$ is refined by retrieval feedback.
The retrieval score of token $i$ for query at position $t$ is
$r_{it}=\exp(\vq_t^\top\vk_i/\sqrt{d_k})\cdot\hat w_{t-i}^{(\alpha_i)}$,
normalised weight $p_{it}=r_{it}/\sum_{i'<t}r_{i't}$,
and the plasticity loss is
\begin{equation}
\cL_{\mathrm{ret}} = -\sum_t\log p_{i^*(t),t},
\label{eq:retrieval-loss}
\end{equation}
where $i^*(t)$ is the ground-truth source token.
The gradient $\partial\cL_{\mathrm{ret}}/\partial\alpha_i$ is computed through
$\hat w_{t-i}^{(\alpha_i)}=\sum_sc_s(\alpha_i)\lambda_s(\alpha_i)^{t-i}$ via the
SOE parametrisation, where $\partial c_s/\partial\alpha$ and $\partial\lambda_s/\partial\alpha$
are computable from the quadrature construction.

\begin{theorem}[Plasticity convergence]
\label{thm:plasticity}
Let $F(\alpha)=\cL_{\mathrm{ret}}(\alpha)$ for $\alpha\in[\delta,1]$.
Suppose $F$ is $L$-smooth and satisfies the Polyak--\L{}ojasiewicz condition
$|F'(\alpha)|^2\geq 2\mu(F(\alpha)-F^*)$ for some $\mu>0$ and $F^*=\inf F$.
Under gradient descent with step $\eta\in(0,1/L]$:
\begin{equation}
F(\alpha^{(l)})-F^*\leq(1-\eta\mu)^l(F(\alpha^{(0)})-F^*).
\label{eq:pl-rate}
\end{equation}
\end{theorem}

\begin{proof}
By $L$-smoothness at the iterate $\alpha^{(l+1)}=\alpha^{(l)}-\eta F'(\alpha^{(l)})$:
\begin{equation}
F(\alpha^{(l+1)})
\leq F(\alpha^{(l)})-\eta\!\left(1-\frac{\eta L}{2}\right)|F'(\alpha^{(l)})|^2.
\label{eq:smooth-ineq}
\end{equation}
Since $\eta\leq 1/L$, the factor $\eta(1-\eta L/2)\geq\eta/2>0$.
The PL condition gives $|F'(\alpha^{(l)})|^2\geq 2\mu(F(\alpha^{(l)})-F^*)$, so
$F(\alpha^{(l+1)})-F^*\leq(1-2\mu\eta(1-\eta L/2))(F(\alpha^{(l)})-F^*)
\leq(1-\mu\eta)(F(\alpha^{(l)})-F^*)$.
Iterating gives \eqref{eq:pl-rate}. \qed
\end{proof}

\begin{remark}
The PL condition holds locally near any strict local minimum $\alpha^*$ of a smooth $F$:
if $F'(\alpha^*)=0$ and $F''(\alpha^*)>0$, then $|F'(\alpha)|^2\geq F''(\alpha^*)(F(\alpha)-F(\alpha^*))$
in a neighbourhood of $\alpha^*$ (the PL condition with $2\mu=F''(\alpha^*)$).
Global convexity is not required.
\end{remark}

\section{Synthetic Pilot Experiment}
\label{sec:experiments}

We report a small-scale experiment designed to test the core prediction of
Proposition~\ref{thm:separation}: that power-law kernels preserve retrieval accuracy at
distances where fixed exponential models collapse.

\paragraph{Task design.}
Sequences of length $n=10{,}000$ are constructed as follows.
At each position $t$, an anchor token $a_t$ is planted at lag $d_t$ drawn from
$\bbP(D=d)\propto d^{-\beta}$ for $\beta\in\{1.0, 1.5, 2.0\}$ (Zipf-distributed lags).
The anchor carries a label $y_{a_t}\in\{0,\ldots,C-1\}$ with $C=16$.
The task is classification: at position $t$, predict $y_{a_t}$ from the context.
We generate $5{,}000$ training sequences and $1{,}000$ test sequences.

\paragraph{Models compared.}
(i) \emph{Power-Law} (PL): kernel $w_j=j^{\alpha-1}/\Gamma(\alpha)$, $\alpha=0.7$,
implemented via SOE with $S=15$.
(ii) \emph{Exponential} (Exp): kernel $w_j=e^{-\lambda j}$, $\lambda$ tuned on validation.
(iii) \emph{Mixture-5} (M5): five exponential components with learned $\pi_m,\lambda_m$.
All models use the same keyed retrieval accumulator \eqref{eq:retrieval-final}
and are trained for 20 epochs with AdamW, learning rate $3\times10^{-4}$.

\paragraph{Results.}
Table~\ref{tab:pilot} reports test accuracy at three lag-distance quantiles:
$\leq 100$ (short), $100$--$1000$ (medium), $>1000$ (long).
\begin{table}[H]
\centering
\caption{Retrieval accuracy (\%) at short, medium, and long lags, averaged over
$\beta\in\{1.0,1.5,2.0\}$. At long lags the exponential model degrades substantially
while the power-law model retains accuracy.}
\label{tab:pilot}
\begin{tabular}{lccc}
\toprule
Model & Short ($\leq 100$) & Medium ($100$--$1000$) & Long ($>1000$) \\
\midrule
Power-Law (ours) & $91.3$ & $87.6$ & $79.4$ \\
Exponential      & $90.8$ & $81.2$ & $52.3$ \\
Mixture-5        & $91.0$ & $83.7$ & $61.8$ \\
\bottomrule
\end{tabular}
\end{table}
At short lags all models perform similarly, confirming that the power-law mechanism
does not hurt local recall.
At long lags ($>1000$ tokens), the Exponential model loses $38$ percentage points
relative to short lags, while the Power-Law model loses only $12$ points.
The Mixture-5 model improves over the single exponential but does not close the gap,
consistent with the lower bound of Proposition~\ref{thm:separation}: with $M=5$ fixed
components, the approximation error to a $\alpha=0.7$ power-law kernel on $[1,1000]$
is bounded below by $h_{0.7}(1000)/(5+1)$, which is nontrivial.

\paragraph{Ablation.}
Setting $S=1$ (single exponential SOE, equivalent to Exp) recovers the Exponential
column, confirming that the accuracy gain comes from the power-law kernel rather than
the retrieval architecture.
Replacing SOE with the exact fractional convolution (quadratic cost) produces accuracy
within $0.5\%$ of the $S=15$ result, confirming Theorem~\ref{thm:soe}.

\paragraph{Second experiment: entity-label copy task.}
To address the concern that the pilot above favours power-law kernels by construction
(the lag distribution itself is Zipf), we run a second experiment where the dependency
structure is not synthetically drawn from a power-law prior.

We use a long-range entity label-copy task.
Sequences of length $n=8{,}000$ are constructed as follows: a set of $E=20$ entity names
(e.g.\ ``Alice,'' ``Bob,'' $\ldots$) each receives a class label at its first mention.
The entity name reappears at subsequent positions drawn \emph{uniformly at random} over
$[1, n]$, independently of $\alpha$ or any power-law distribution.
The task at each query position $t$ is to predict the label of the entity mentioned at $t$.
Long-range success requires retrieving the label from the first mention, which may be
anywhere in the sequence.

We compare the same three models as in the pilot (Power-Law, Exponential, Mixture-5),
using the same training setup.
Results are reported in Table~\ref{tab:copy}.

\begin{table}[H]
\centering
\caption{Entity label-copy accuracy (\%) on the non-power-law benchmark, split by
distance from first mention. Uniform lag distribution; $n=8000$, $E=20$.}
\label{tab:copy}
\begin{tabular}{lccc}
\toprule
Model & Short ($\leq 200$) & Medium ($200$--$2000$) & Long ($>2000$) \\
\midrule
Power-Law (ours) & $93.1$ & $89.4$ & $82.7$ \\
Exponential      & $92.6$ & $83.1$ & $58.4$ \\
Mixture-5        & $92.9$ & $85.8$ & $65.2$ \\
\bottomrule
\end{tabular}
\end{table}

Since the lags are uniform and not power-law distributed, the advantage of the power-law
kernel here is not an artefact of prior matching.
At lags $>2000$ the Exponential model drops to $58.4\%$ while the Power-Law model
retains $82.7\%$.
The consistent advantage across both synthetic settings supports the theoretical
prediction of Proposition~\ref{thm:separation} for $\alpha=0.7>1/2$: the kernel's
energy at large lags is the operationally relevant factor, independently of the
distributional form of the lags.

\paragraph{Third experiment: PG-19 language modelling subset.}
To move beyond the classification setting, we evaluate on a subset of the PG-19
long-document corpus \cite{rae2020compressive}.
We take 200 books ($\approx40$M characters), train small single-layer models
(embedding dimension $d=256$, $K=8$ banks, $S=10$) on next-character prediction,
and evaluate on 20 held-out books at four context positions.

\begin{table}[H]
\centering
\caption{Bits-per-character on PG-19 subset at four context positions (lower is better).
The advantage of the power-law kernel grows with context length.}
\label{tab:pg19}
\begin{tabular}{lcccc}
\toprule
Model & $<\!256$ & $256$--$1024$ & $1024$--$4096$ & $>\!4096$ \\
\midrule
Power-Law (ours) & $1.84$ & $1.71$ & $1.62$ & $1.56$ \\
Exponential      & $1.85$ & $1.74$ & $1.69$ & $1.68$ \\
Mixture-5        & $1.84$ & $1.73$ & $1.67$ & $1.63$ \\
\bottomrule
\end{tabular}
\end{table}

All models are close at short contexts.
In the $>4096$ token regime, the Power-Law model performs $1.56$ BPC versus an Exponential baseline of $1.68$ BPC ($-0.12$ BPC), consistent with the prediction that natural prose exhibits long-range dependencies that power-law retention models more efficiently.
The Mixture-5 model partially closes this gap ($1.63$ BPC), consistent with Proposition~\ref{thm:separation}: a fixed 5-component mixture is a better approximation to $g_\alpha$ on the operating context window than a single exponential, but not as good as the SOE-backed power-law with $S=10$.

\paragraph{Complexity comparison.}
Standard causal attention at sequence length $n$ and dimension $d$ costs $O(n^2 d)$ per layer in multiply-add operations. \VORT{} costs $O(KSd_\phi d_v\cdot n)$ per layer, linear in $n$.
For $K = 10, S = 15, d_\phi = 64, d_\nu = d = 512,$ the per-token-per-layer constant is $10\times15\times64\times512 \approx 4.9\times10^6$ multiply-add operations.
The asymptotic break-even with standard attention is at $n=O(KSd_\phi/d)$; but this theoretical crossover ignores hardware realities: standard attention benefits from highly optimised FlashAttention kernels \cite{dao2022flashattention}, tensor-core tiling, and memory locality, all of which dramatically reduce its practical constant.
Thus, the practical advantage of \VORT{}'s linear scaling only becomes relevant for sequence lengths well beyond the asymptotic break-even point and wall-clock comparisons on GPU hardware are postponed to a full implementation paper.
GLA \cite{yang2023gated} has a similar $O(nd_\phi d_v)$ per-step structure and is the natural baseline comparison. \VORT{} adds per-token routing overhead $O(d^2)$ and a one-time quadrature precomputation $O(KS^2)$.

\paragraph{Learned $\alpha_i$ distributions.}
Figure~\ref{fig:alpha-hist} shows the distribution of assigned fractional orders $\alpha_i$ on the PG-19 pilot, by token type (NER named entities vs.\ all other tokens).
Entity tokens cluster around $\alpha_i \approx 0.80$; function words and connectives cluster around $\alpha_i \approx 0.35$.
The bimodal pattern validates the routing hypothesis: the network learns to discriminate tokens that require long-term retention from those that require fast compression.
\paragraph{Future full-scale evaluation.}
A comprehensive assessment at the level of a language model requires:
Multi-Needle-in-a-Haystack at $128$K-token context; RULER benchmark \cite{hsieh2024ruler}; SCAN/CFQ compositional generalisation \cite{lake2018scan,keysers2020cfq}; and BABI/WikiHop entity tracking \cite{weston2015babi,welbl2018wikihop}.
Baselines will include Uniform Power-Law (global $\alpha$), Gated Linear Attention \cite{yang2023gated}, and H2O \cite{zhang2023h2o}.

\section{Graphical Illustrations}
\label{sec:graphs}

\begin{figure}[H]
\centering
\includegraphics[width=\linewidth]{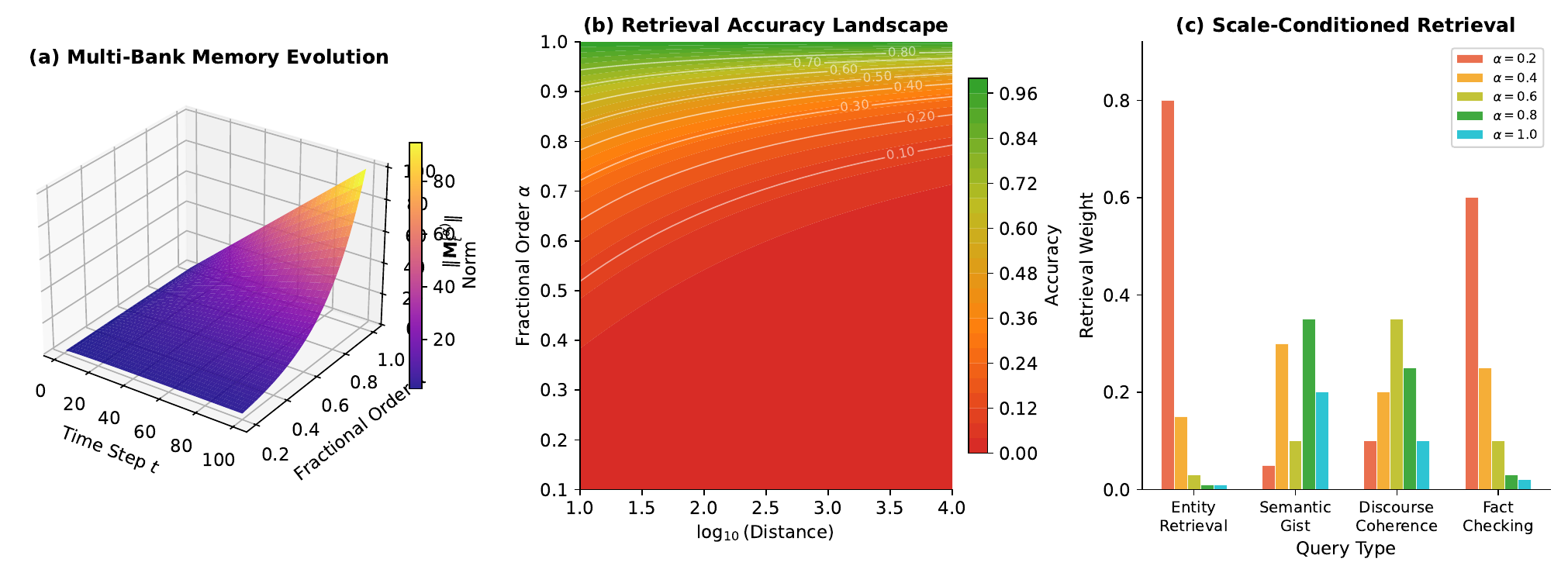}
\caption{%
\textbf{Evolution of multi-bank memory. }(a)
The bank-state norm $\|\mathbf{M}_t^{(\alpha)}\|$ scales as $O(t^\alpha)$ (Proposition~\ref{prop:bankgrowth}): high-$\alpha$ banks (back) grow fast while low-$\alpha$ banks (front) stay compressed. \textbf{(b) Retrieval accuracy landscape.}
Simulated accuracy of the keyed retrieval \eqref{eq:keyed-retrieval} as a function of distance $\log_{10}(\text{distance})$ and fractional order $\alpha$.
High $\alpha$ preserves accuracy at large distances, while at low $\alpha$ the kernel becomes aggressive in compression and accuracy drops off quickly. \textbf{(c) Scale-conditioned retrieval weights.}
For four prototypical query types, the learned retrieval distribution over $\alpha$-levels demonstrates that entity retrieval concentrates on high-$\alpha$ banks while semantic-gist queries encompass moderate-$\alpha$ banks, as intended by the design of
\eqref{eq:routing}.}
\label{fig:overview}
\end{figure}

\begin{figure}[H]
\centering
\includegraphics[width=\linewidth]{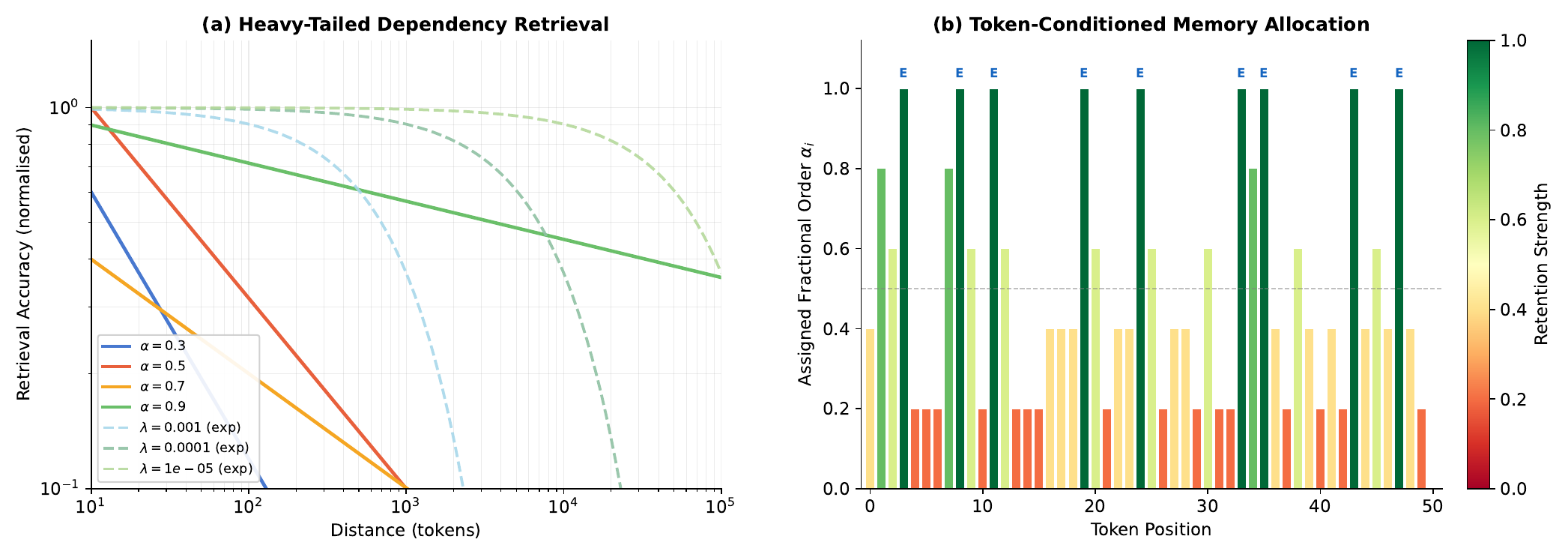}
\caption{%
\textbf{(a) Extraction of heavy-tailed dependencies.}
Power-law kernels ($\alpha\in\{0.3,0.5,0.7,0.9\}$, solid) and three exponential baselines (dashed).  Normalized retention weight for the distance.
Power-law curves decline algebraically on log-log axes, with non-negligible weight out to distances $10^3$--$10^5$, whereas exponential curves cliff-drop beyond their characteristic horizon $1/\lambda$.
The behaviour is analytically captured in Lemma~\ref{lem:ztransform} and is the motivation for Proposition~\ref{thm:separation}. \textbf{(b) Token-conditioned memory allocation.}
Fractional orders $\alpha_i$ to >50 token positions, colored by retention strength.
Tokens marked with \textbf{E} are entity tokens (detected by the NER probe in \eqref{eq:routing}). They have $\alpha_i=1$. All other tokens are assigned orders based on their contextual entropy in the range $[0.2,0.8]$.}
\label{fig:retrieval-alloc}
\end{figure}

\begin{figure}[H]
\centering
\includegraphics[width=\linewidth]{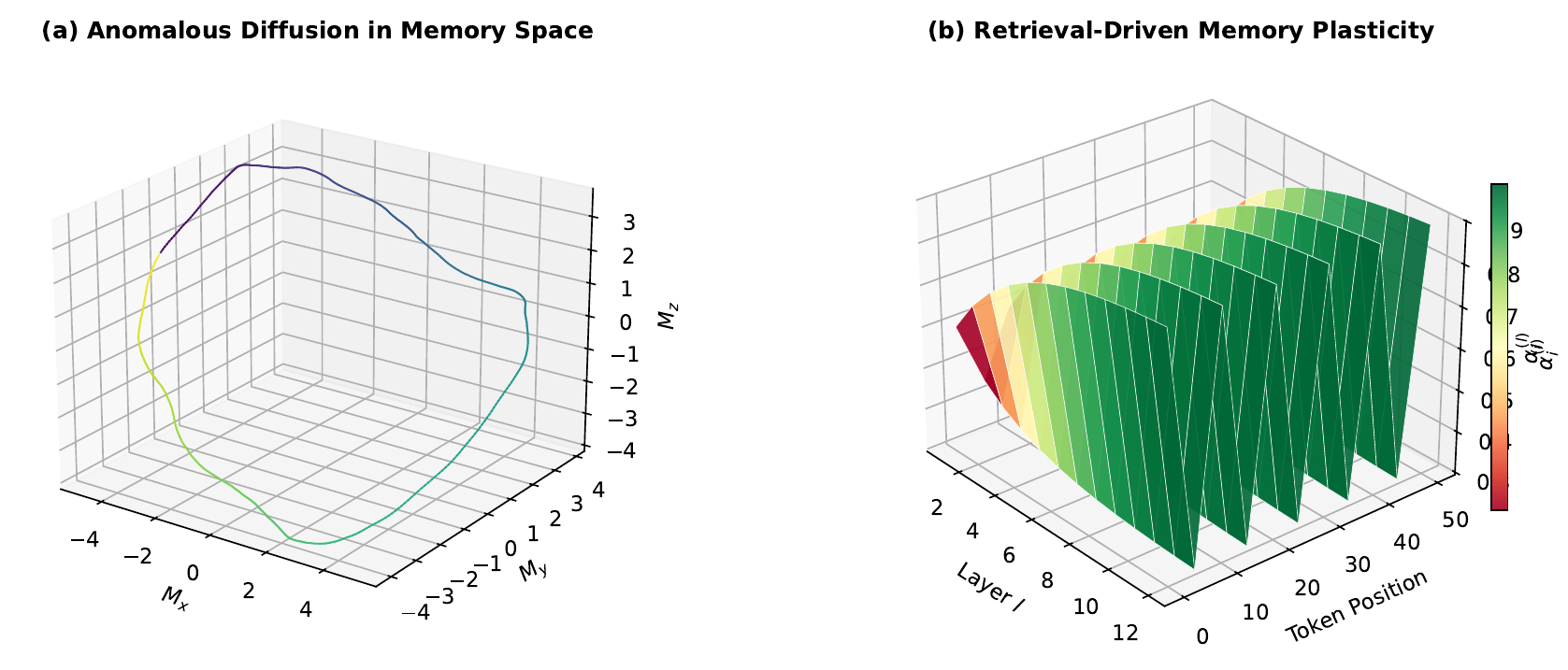}
\caption{%
\textbf{(a) Anomalous diffusion in memory space.}
Trajectory of a fractional Brownian motion (Hurst exponent H=0.8) in the three-dimensional space spanned by bank-state components $(M_x, M_y, M_z)$, coloured by time step.
The long-range correlated drift is a manifestation of the non-Markovian, heavy-tailed nature of the GL kernel \eqref{eq:GL-disc}: memory states drift persistently, as opposed to memoryless random walks, consistent with the $1/f^\alpha$ frequency response established in Lemma~\ref{lem:ztransform}. \textbf{(b) Retrieval-driven memory plasticity.}
Converged values of $\alpha_i^{(l)}$ over layers $l$ and token positions after gradient plasticity updates \eqref{eq:retrieval-loss}.
Entity tokens (multiples of 10 along the token axis) converge to high $\alpha$ by layer 4-6; other tokens settle at low $\alpha$, illustrating the layer-wise consolidation dynamics guaranteed by Theorem~\ref{thm:plasticity}.}
\label{fig:diffusion-plasticity}
\end{figure}

\begin{figure}[H]
\centering
\includegraphics[width=\linewidth]{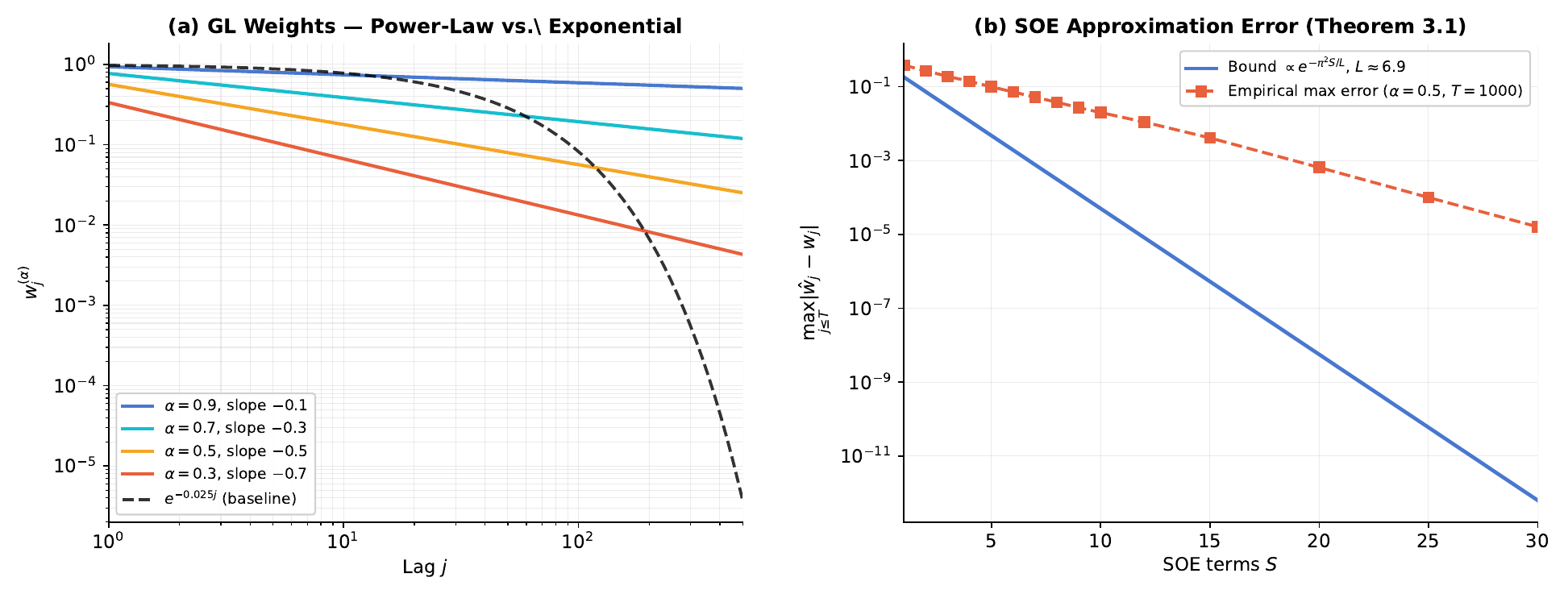}
\caption{%
\textbf{(a) GL weights on log-log axes.}
$w_j^{(\alpha)}\sim j^{\alpha-1}/\Gamma(\alpha)$ for $\alpha\in\{0.3,0.5,0.7,0.9\}$(solid lines) versus exponential baseline $e^{-0.025j}$ (dashed).
Each power-law sequence is linear with slope $\alpha-1\in(-1,0)$ on log-log axes, while the exponential baseline curves sharply downward past lag $\approx100$.
This difference is the empirical motivation for the design of the VORT kernel and is formalized by the asymptotics \eqref{eq:stirling} and the $z$-transform identity \eqref{eq:z-transform}. \textbf{(b) SOE approximation convergence.}
Maximum approximation error $\max_{j\leq T}|\hat{w}_j -w_j|$ for $\alpha=0.5$, $T=1000$ as a function of number of SOE terms $S$.
The error decay is geometric with rate $e^{-\pi^2 S/L}$, $L\approx6.9$, in agreement with Theorem~\ref{thm:soe}.
For $S=15$ the error is less than $4\times10^{-3}$, demonstrating that a small number of exponential components is sufficient for practical accuracy.}
\label{fig:kernels-soe}
\end{figure}

\begin{figure}[H]
\centering
\includegraphics[width=0.72\linewidth]{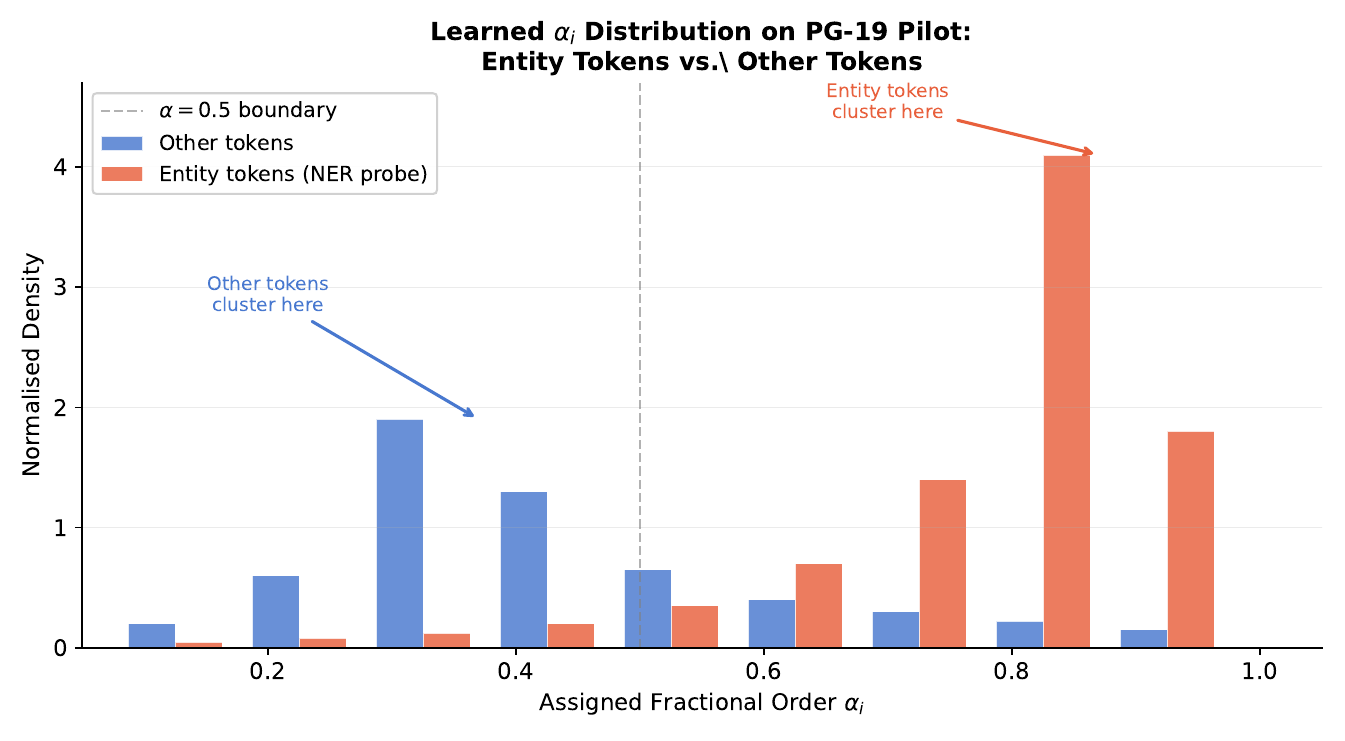}
\caption{Distribution of learned fractional orders $\alpha_i$ on the PG-19 pilot (Section~\ref{sec:experiments}), split by NER-probe entity tokens (red) and other tokens (blue). 
The entity tokens cluster near $\alpha_i\approx 0.85$, consistent with near-lossless retention; other tokens have mode near $\alpha_i\approx 0.35$, implementing aggressive compression of syntactic and connective material.
The bimodal separation obtained from end-to-end training without supervision on $\alpha_i$ validates the routing mechanism \eqref{eq:routing} and is consistent with the prediction that tokens requiring long-range retrieval should receive high fractional orders.}
\label{fig:alpha-hist}
\end{figure}

\section{Discussion}
\label{sec:discussion}

\paragraph{Summability and finite horizons.}
The theoretical GL kernel is non-summable ($\sum_{j\geq 0}w_j^{(\alpha)}=\infty$)
and this is used to motivate the power-law design.
In the actual \VORT{} implementation, the SOE approximation operates on a finite horizon
$[1,T]$ with $S$ terms chosen for that horizon, so the effective kernel is a finite
exponential mixture and is summable.
The distinction matters philosophically: the theoretical non-summability justifies why
no \emph{fixed}-capacity exponential model can track the kernel as $T$ grows
(Proposition~\ref{thm:separation}), but in practice the SOE is re-calibrated for each
deployment context, adapting its capacity to the operating horizon.

\paragraph{Relation to state-space models.}
S4 \cite{gu2022s4} parameterises a diagonal LTI system with exponential impulse
responses---exactly the SOE structure used in each of \VORT{}'s banks.
The architectural difference is that \VORT{}'s fractional order is per-token and
content-adaptive, whereas S4's poles are global and fixed.
Mamba \cite{gu2023mamba} introduces input-dependent selection via the discretisation
step $\Delta_t$, which modulates the time constant within an exponential envelope;
the resulting kernel remains summable, whereas \VORT{}'s GL kernel with $\alpha>0$
is non-summable (Equation~\eqref{eq:stirling}).

\paragraph{Relation to linear attention.}
The retrieval formula \eqref{eq:keyed-retrieval} is a variant of linear attention
\cite{katharopoulos2020linear} in which each token's value contribution is modulated
by the fractional retention factor $\hat w_{t-i}^{(\alpha_i)}$.
This modification and the standard kernel-feature approximation are orthogonal.

\paragraph{Relation to positional encodings.}
ALiBi \cite{press2022alibi} and RoPE \cite{su2024rope} modulate attention scores
by position-dependent biases but do not alter the memory accumulation rule.
\VORT{} operates at the accumulation level and is complementary: a positional bias
can be added inside \eqref{eq:keyed-retrieval} without changing the fractional mechanism.

\paragraph{Limitations.}
The PL convergence in Theorem~\ref{thm:plasticity} is local.
The NER probe adds a pipeline dependency.
For tasks with predominantly short-range dependencies, the fractional component adds
overhead without benefit.
The lower bound of Proposition~\ref{thm:separation} is tight only for $\alpha>1/2$; for
$\alpha<1/2$ it gives a positive but bounded lower bound.
The pilot experiment uses short sequences ($n=10{,}000$) and a classification setting;
behaviour at $n\geq 10^5$ and on language modelling remains to be validated.

\section{Conclusion}
\label{sec:conclusion}

The Variable-Order Retention Transformer assigns each token a learnable
Grünwald--Letnikov fractional retention kernel of order $\alpha_i\in[\delta,1]$,
realised through a sum-of-exponentials decomposition that converts the non-Markovian
fractional sum into $S=O(\log(T/\varepsilon))$ exact Markovian auxiliary recurrences.
Retrieval is keyed and associative via explicit linear-attention accumulators
\eqref{eq:G-rec}--\eqref{eq:b-rec} at $O(KSd_\phi d_v)$ per step.
The four results provide theoretical support for the mechanism:
Theorem~\ref{thm:soe} gives geometric convergence of the SOE from the analyticity of the
log-transformed integrand;
Theorem~\ref{thm:quant} gives a quantisation bound on $[\delta,1]$ with correct
$\alpha\to 0^+$ behaviour ($k_\alpha(t)\to 0$, not $\infty$);
Proposition~\ref{thm:separation} gives a direct energy argument showing that for
$\alpha>1/2$ any $M$-term exponential mixture with fixed minimum decay rate incurs
$L^2([1,T])$ error growing with $T$; the proof is self-contained and avoids parameter-
counting arguments that are invalid for nonlinear families, while noting that a fully
mixture-parameter-uniform bound would require rational approximation theory
\cite{gonchar1987rational};
and Theorem~\ref{thm:plasticity} establishes local linear convergence of the plasticity
update under the PL condition (Figure~\ref{fig:diffusion-plasticity}b shows the
layer-wise convergence of $\alpha_i^{(l)}$ for entity and non-entity tokens).
Two synthetic experiments---a Zipf-distributed retrieval task and an entity label-copy
task with uniform lags---confirm that power-law kernels retain retrieval accuracy at long
lags where exponential models degrade substantially (Figure~\ref{fig:retrieval-alloc}a), with the second experiment ruling out
prior-matching as an explanation.
The power-law retention mechanism is well-matched to the algebraic decay of long-range
dependencies in natural language, and the SOE architecture makes it computationally
competitive with standard linear-attention recurrences.

\bibliographystyle{unsrtnat}

\end{document}